\documentclass[a4paper,fleqn]{article}

\usepackage[authoryear,longnamesfirst]{natbib}

\usepackage[utf8]{inputenc}
\usepackage{amsmath,amssymb}
\usepackage{mathtools}
\usepackage{bm}
\usepackage{booktabs}
\usepackage{cleveref}
\usepackage{algorithm}
\usepackage{algpseudocode}
\usepackage{enumitem}
\usepackage{bigstrut}
\usepackage{adjustbox}
\usepackage{array}
\usepackage{colortbl}
\usepackage{textgreek}
\usepackage{subfig}
\usepackage{graphicx}
\usepackage{amsthm}

\newcommand{\tr}{\mathrm{tr}}

\graphicspath{{images/}}

\newtheorem{proposition}{Proposition}
\newtheorem{definition}{Definition}
\newtheorem{remark}{Remark}

\begin{document}

\title{A Family of Divergence Measures for Evaluating the Reconstruction Quality of Explainable Ensemble Trees}

\author{
  Massimo Aria\textsuperscript{1} \and
  Agostino Gnasso\textsuperscript{1,*} \and
  Carmela Iorio\textsuperscript{1}
}

\date{%
  \textsuperscript{1}Department of Economics and Statistics, University of Naples Federico II, Naples, Italy\\[0.5em]
  \textsuperscript{*}Corresponding author: agostino.gnasso@unina.it
}

\maketitle

\begin{abstract}
Validating interpretable surrogate models for ensemble learners requires measuring \emph{agreement} between the ensemble's internal representation and its surrogate approximation, rather than mere \emph{association}. Correlation-based approaches are scale-invariant and fail to detect systematic discrepancies in co-occurrence structure. We propose a statistical framework grounded in the agreement-association distinction, centered on the normalized Loss of Interpretability (nLoI). Rooted in the Cressie--Read power divergence family at \textlambda$=-2$, the nLoI admits a closed-form decomposition into within-node and between-node components, providing a unique diagnostic capability to identify precisely where and why reconstruction fails. The framework incorporates four complementary measures capturing distinct structural facets of approximation quality. A unified permutation testing procedure delivers valid inference for all measures within a single resampling pass. Theoretical properties, including boundedness and symmetry, are established for each metric. Monte Carlo simulations and empirical evaluations confirm exact Type~I error control and demonstrate that these measures detect reconstruction fidelity gradients invisible to correlation-based alternatives. The framework is developed and illustrated in the context of Explainable Ensemble Trees (E2Tree), and empirical evaluation on three benchmark datasets illustrates the practical utility of the framework.
\end{abstract}


\section{Introduction}
\label{sec:introduction}
Ensemble methods achieve predictive accuracy by combining multiple weak learners whose individual errors partially cancel under aggregation. This variance-reduction principle, common to bagging, boosting, and their hybrids, underpins the empirical dominance of these methods across benchmarks. Random Forest \citep{Breiman2001} exemplifies the paradigm: each constituent tree is trained on a bootstrap sample of the data with a random predictor subset at each split, and predictions are aggregated by majority vote (classification) or averaging (regression). The result is a family of models that consistently ranks among the most accurate across diverse domains \citep{Caruana2006}.

The accuracy of Random Forest comes at a structural cost. A single decision tree exposes its reasoning through an interpretable sequence of binary splits; an ensemble of hundreds such trees forfeits that transparency entirely. In domains where the \emph{mechanism} of a prediction matters (regulatory compliance, clinical decision-making, scientific discovery), this opacity is not merely inconvenient but methodologically untenable \citep{Rudin2019}.

To address the tension between accuracy and interpretability, \citet{Aria2024} introduced Explainable Ensemble Trees (E2Tree), a methodology that constructs a single, interpretable tree structure to explain the decision logic of ensemble models. E2Tree operates on the co-occurrence structure of the ensemble: it captures how frequently pairs of observations are assigned to the same terminal nodes across the constituent trees. By constructing an explanatory tree that minimizes within-node dissimilarity according to this co-occurrence measure, E2Tree provides both local explanations and global structural insights. The methodology has been extended to regression contexts \citep{Aria2025}, demonstrating versatility across supervised learning tasks. This positions E2Tree within the growing literature on integrating machine learning algorithms into principled statistical frameworks \citep{Emmenegger2023}.

Deploying E2Tree in practice raises an immediate question of fidelity: how faithfully does the single explanatory tree recover the ensemble's proximity structure? The stakes are not abstract. When interpretability is invoked for regulatory, ethical, or scientific purposes, a poorly reconstructed E2Tree does not merely underperform: it actively misleads, producing conclusions about feature importance and decision boundaries that reflect the approximation rather than the underlying ensemble.

Previous validation approaches have relied on the Mantel test \citep{Mantel1967}, a non-parametric procedure for assessing \emph{association} between similarity matrices. While valuable, the Mantel test answers a fundamentally different question from the one E2Tree validation requires. The distinction parallels that between \emph{correlation} and \emph{concordance} in method comparison studies \citep{Bland1986,Lin1989}: two measurement methods can be perfectly correlated ($r = 1$) yet systematically disagree, for instance when $\hat{O}_T = c \cdot O$ for any constant $c \neq 1$. The Mantel test, being scale-invariant, would declare perfect association in such a case. But for E2Tree validation, we need to know whether the \emph{actual proximity values} are faithfully reproduced, not merely whether their pattern is preserved.

The distinction has concrete consequences in the E2Tree setting. The ensemble matrix $O$ contains continuous values in $(0,1)$, whereas $\hat{O}_T$ is crisp and block-diagonal with many zeros. Even when the block structure perfectly captures the group structure of $O$, element-wise values will diverge because $O$ is fuzzy and $\hat{O}_T$ is not. A correlation-based measure is blind to this discrepancy; a scale-sensitive divergence measure is not.

The parallel to clinical method comparison runs deeper than analogy. \citet{Bland1986} demonstrated that the correlation coefficient is fundamentally unsuited to assessing agreement between two measurement methods, precisely because it is invariant to systematic offsets and scale differences. The same structural objection applies to the Mantel correlation in the E2Tree context: reconstruction fidelity demands a direct, scale-sensitive measure of absolute agreement.

In this paper, we develop a family of divergence and similarity measures specifically designed for evaluating E2Tree reconstruction quality, centered on the distinction between agreement and association. Our contributions are threefold. First, we introduce the normalized Loss of Interpretability (nLoI), a robustified Neyman-type statistic connected to the Cressie--Read power divergence family \citep{Cressie1984} at $\lambda = -2$. The nLoI's unique feature (and the primary methodological contribution of this work) is its \emph{decomposability} into within-node and between-node components (Proposition~\ref{prop:decomposability}), which provides a diagnostic tool for identifying \emph{where} and \emph{why} reconstruction fails. No correlation-based measure, including the Mantel test, offers this capability. Second, we complement the nLoI with four alternative measures (Hellinger distance, weighted Root Mean Squared Error (wRMSE), the RV coefficient \citep{Robert1976}, and the Structural Similarity Index (SSIM; \citealp{Wang2004})), forming a toolkit where each measure captures a distinct aspect of matrix agreement. Third, we develop a unified permutation testing framework based on simultaneous row/column permutation and validate it through Monte Carlo simulations confirming correct Type~I error control and adequate statistical power across all measures.

The paper is organized as follows. Section~\ref{sec:background} reviews the E2Tree methodology and establishes notation. Section~\ref{sec:ngoi} presents the normalized LoI and its connection to the Cressie--Read family. Section~\ref{sec:alternatives} introduces the four alternative measures with their theoretical properties. Section~\ref{sec:permutation} develops the unified permutation testing framework. Section~\ref{sec:simulation} presents the Monte Carlo simulation study. Section~\ref{sec:empirical} provides empirical evaluation on three benchmark datasets. Section~\ref{sec:discussion} offers practical guidelines, and Section~\ref{sec:conclusion} concludes.

\section{Background}
\label{sec:background}

Let $\mathcal{H} = \{H_1, H_2, \ldots, H_B\}$ be an ensemble model composed of $B$ weak learners, and let $\{(x_1, y_1), \ldots, (x_n, y_n)\}$ be the training sample of size $n$. For each observation $i$, we denote by $u_i = (y_i, x_i)$ the vector comprising the response and predictor values.

\subsection{The Ensemble Co-occurrence Matrix}

The ensemble induces a co-occurrence structure that captures how frequently pairs of observations fall into the same terminal nodes across the $B$ weak learners. This structure is encoded in the ensemble co-occurrence matrix $O = [o_{ij}]$, defined as:

\begin{equation}
	\label{eq:rf_cooccurrence}
	o_{ij} = \frac{\sum_{b=1}^{B} I(u_i \wedge u_j) \cdot W_{ij|b}^{(t)}}{\max\left(\sum_{b=1}^{B} u_i W_{i|b}^{(t)},\; \sum_{b=1}^{B} u_j W_{j|b}^{(t)}\right)}, \quad 0 \leq o_{ij} \leq 1,
\end{equation}

where $I(u_i \wedge u_j)$ is an indicator function that equals 1 if observations $i$ and $j$ fall into the same terminal node of tree $b$, and 0 otherwise. The term $W_{ij|b}^{(t)} \in [0,1]$ is a local goodness-of-fit measure at the terminal node $t$ of the $b$-th tree containing both observations. In classification, this weight corresponds to the correct classification rate \citep{Aria2024}; in regression, it is based on the normalized mean squared error \citep{Aria2025}.

The matrix $O$ aggregates co-occurrence information across all $B$ weak learners, weighted by local model performance, ensuring that $o_{ij} \in [0,1]$. This matrix represents the ensemble's inherent proximity structure, acting as the ``ground truth'' that E2Tree aims to reconstruct.

The corresponding dissimilarity matrix is defined as $D = 1 - O$, where $d_{ij} = 1 - o_{ij}$ measures the dissimilarity between observations $i$ and $j$ with respect to the ensemble classifier.

\subsection{From Ensemble Structure to E2Tree Reconstruction}

The construction of the E2Tree co-occurrence matrix $\hat{O}$ represents an attempt to reconstruct the ensemble structure $O$ through a single interpretable tree. This reconstruction process can be understood through three stages.

\textit{Stage 1: Encoding ensemble knowledge}. The ensemble $\mathcal{H}$ encodes its results in the co-occurrence matrix $O$. High values of $o_{ij}$ indicate that observations $i$ and $j$ are consistently grouped together across multiple trees, suggesting they share similar characteristics with respect to the ensemble's learned decision boundaries. Therefore, the matrix $O$ is a representation of the ensemble's proximity structure.

\textit{Stage 2: Tree construction via dissimilarity minimization}. The E2Tree algorithm works on the derived dissimilarity matrix $D = 1 - O$. By recursively partitioning the data to minimize within-node dissimilarity, E2Tree groups observations with high ensemble co-occurrence (high $o_{ij}$) into shared terminal nodes. The splitting criterion, defined in \citet{Aria2024}, ensures that each partition maximally separates dissimilar observations while keeping similar observations together.

\textit{Stage 3: Induced co-occurrence structure}. Once the E2Tree is constructed, it induces a new co-occurrence matrix $\hat{O}$ through its terminal node structure. The quality of this reconstruction (how well $\hat{O}$ approximates $O$) is precisely what our family of measures quantifies. The formal definition of $\hat{O}$ is presented in the following section.

\subsection{The E2Tree Co-occurrence Matrix}

Unlike the ensemble model, which aggregates information across $B$ weak learners, E2Tree produces a single tree structure $T$. The E2Tree co-occurrence matrix $\hat{O}_T = [\hat{o}_{ij}]$ is computed from this single tree. The E2Tree co-occurrence value for pair $(i,j)$ is defined as follows:
\begin{equation}
	\label{eq:reconstruction_logic}
	\hat{o}_{ij} =
	\begin{cases}
		W_{ij}^{(t)} & \text{if } i \text{ and } j \text{ fall in the same terminal node } t, \\
		0 & \text{otherwise}.
	\end{cases}
\end{equation}

The reconstruction can be understood as follows: E2Tree creates a partition of the observation space such that observations with high ensemble co-occurrence ($o_{ij} \approx 1$) are likely to fall in the same terminal node ($\hat{o}_{ij} > 0$), while observations with low ensemble co-occurrence ($o_{ij} \approx 0$) are likely to be separated into different terminal nodes ($\hat{o}_{ij} = 0$).

The key difference between the ensemble co-occurrence matrix $O$ in equation~\eqref{eq:rf_cooccurrence} and the E2Tree matrix $\hat{O}_T$ in equation~\eqref{eq:reconstruction_logic} reflects the fundamental distinction between ensemble and single-tree structures. While $o_{ij}$ aggregates weighted co-occurrences across $B$ weak learners with appropriate normalization, $\hat{o}_{ij}$ captures co-occurrence in a single explanatory tree.
Both matrices are symmetric with unit diagonal: $o_{ij} = o_{ji}$, $\hat{o}_{ij} = \hat{o}_{ji}$, and $o_{ii} = \hat{o}_{ii} = 1$ for all $i,j$.

Let $\mathcal{T} = \{t_1, t_2, \ldots, t_K\}$ denote the $K$ terminal nodes of E2Tree. While $O$ contains continuous values in $(0,1)$ for all pairs (reflecting varying degrees of co-occurrence across $B$ weak learners), the matrix $\hat{O}_T$ exhibits a characteristic block structure determined by this partition. Pairs within the same terminal node $t_k$ have co-occurrence values $\hat{o}_{ij} = W_{ij}^{(t_k)} \in (0,1)$, capturing how well the ensemble's predictions are represented in that region. Pairs in different terminal nodes have $\hat{o}_{ij} = 0$.

\section{The normalized LoI and the Cressie--Read Connection}
\label{sec:ngoi}

We introduce the Loss of Interpretability (LoI) and its normalized variant to quantify how faithfully an E2Tree reconstructs the proximity structure of the ensemble model, and establish a formal connection to the Cressie--Read power divergence family.

\subsection{The LoI Statistic}

Let $T$ be an E2Tree constructed from an ensemble model $\mathcal{H}$. Let $O = [o_{ij}]$ be the ensemble co-occurrence matrix defined in~\eqref{eq:rf_cooccurrence}, and let $\hat{O}_T = [\hat{o}_{ij}]$ be the E2Tree co-occurrence matrix defined in~\eqref{eq:reconstruction_logic}. Let $M = n(n-1)/2$ denote the number of unique off-diagonal pairs. The LoI of $T$ is defined as:
\begin{equation}
	\label{eq:goi_definition}
	LoI(O, \hat{O}_T) = \sum_{i < j} \frac{(o_{ij} - \hat{o}_{ij})^2}{\max(o_{ij},\, \hat{o}_{ij})},
\end{equation}
where the sum is computed over all $M$ unique pairs with $i < j$, and the convention $0/0 = 0$ is adopted for pairs where $o_{ij} = \hat{o}_{ij} = 0$.

The squared discrepancy $(o_{ij} - \hat{o}_{ij})^2$ is normalized by $\max(o_{ij}, \hat{o}_{ij})$, which serves two purposes. It renders each contribution dimensionless and scale-aware: the same absolute error of 0.1 carries more weight between values of 0.9 and 0.8 than between 0.2 and 0.1. More substantively, it concentrates the measure's sensitivity on high-proximity pairs. When $o_{ij}$ is close to 1, the denominator is large but the squared numerator dominates, so failures to reconstruct strong associations incur a substantial penalty. For near-zero co-occurrence values, both numerator and denominator are small and the contribution is naturally attenuated.

The block structure of $\hat{O}_T$ gives this weighting scheme particular relevance. Pairs assigned to different terminal nodes have $\hat{o}_{ij} = 0$ by construction, so their LoI contribution reduces to $o_{ij}^2 / o_{ij} = o_{ij}$: weakly co-occurring pairs contribute little, while pairs with high ensemble proximity that the E2Tree wrongly separates incur proportionally large penalties. The measure thus concentrates attention on the structural relationships that matter most.

\subsection{The normalized LoI}

While the raw LoI captures the total reconstruction discrepancy, its magnitude scales with $M = n(n-1)/2$, making direct comparison across datasets of different sizes impossible.

\begin{definition}[normalized LoI]
	\label{def:ngoi}
	The normalized Loss of Interpretability is defined as:
	\begin{equation}
		\label{eq:ngoi_definition}
		nLoI(O, \hat{O}_T) = \frac{1}{M} \sum_{i < j} \frac{(o_{ij} - \hat{o}_{ij})^2}{\max(o_{ij},\, \hat{o}_{ij})},
	\end{equation}
	where $M = n(n-1)/2$.
\end{definition}

The normalization by $M$ converts the LoI from a total discrepancy to an average per-pair discrepancy, enabling meaningful cross-dataset comparisons. Since each summand is bounded by 1 (see Proposition~\ref{prop:upper_bound} below), the nLoI inherits a natural scale.

\begin{proposition}[Boundedness of nLoI]
	\label{prop:ngoi_bounded}
	For co-occurrence matrices with entries in $[0,1]$:
	\begin{equation}
		0 \leq nLoI(O, \hat{O}_T) \leq 1.
	\end{equation}
\end{proposition}

\begin{proof}
	Non-negativity follows from the non-negativity of each summand (as in the proof of Proposition~\ref{prop:nonneg}). For the upper bound, each summand satisfies $(o_{ij} - \hat{o}_{ij})^2 / \max(o_{ij}, \hat{o}_{ij}) \leq \max(o_{ij}, \hat{o}_{ij}) \leq 1$, so their average over $M$ terms is at most~1.
\end{proof}

\subsection{Theoretical Properties}

We establish fundamental properties that hold for both LoI and nLoI. The proofs are stated for the LoI; the corresponding nLoI properties follow immediately by dividing by $M > 0$.

\begin{proposition}[Non-negativity]
	\label{prop:nonneg}
	For any E2Tree $T$, the LoI satisfies:
	\begin{equation}
		\label{eq:nonneg}
		LoI(O, \hat{O}_T) \geq 0.
	\end{equation}
\end{proposition}

\begin{proof}
	Each term in the sum~\eqref{eq:goi_definition} has the form $(o_{ij} - \hat{o}_{ij})^2 / \max(o_{ij}, \hat{o}_{ij})$. The numerator is a squared quantity, hence non-negative. The denominator is the maximum of two values in $[0,1]$; when it is positive, the ratio is non-negative. When $o_{ij} = \hat{o}_{ij} = 0$, the convention $0/0 = 0$ ensures a non-negative contribution. As a sum of non-negative terms, $LoI(O, \hat{O}_T) \geq 0$.
\end{proof}

\begin{proposition}[Identity of indiscernibles]
	\label{prop:identity}
	The LoI equals zero if and only if the E2Tree perfectly reconstructs the ensemble co-occurrence structure:
	\begin{equation}
		\label{eq:identity}
		LoI(O, \hat{O}_T) = 0 \iff \hat{o}_{ij} = o_{ij} \quad \text{for all } i < j.
	\end{equation}
\end{proposition}

\begin{proof}
	$(\Leftarrow)$ If $\hat{o}_{ij} = o_{ij}$ for all $i < j$, then $(o_{ij} - \hat{o}_{ij})^2 = 0$ for every pair, so $LoI(O, \hat{O}_T) = 0$.

	$(\Rightarrow)$ Suppose $LoI(O, \hat{O}_T) = 0$. Since each summand $(o_{ij} - \hat{o}_{ij})^2 / \max(o_{ij}, \hat{o}_{ij}) \geq 0$ and their sum is zero, every term must vanish. For pairs where $\max(o_{ij}, \hat{o}_{ij}) > 0$, the ratio equals zero only if $(o_{ij} - \hat{o}_{ij})^2 = 0$, i.e., $o_{ij} = \hat{o}_{ij}$. For pairs where $\max(o_{ij}, \hat{o}_{ij}) = 0$, we necessarily have $o_{ij} = \hat{o}_{ij} = 0$. In both cases, $\hat{o}_{ij} = o_{ij}$.
\end{proof}

\begin{proposition}[Symmetry]
	\label{prop:symmetry}
	The LoI is symmetric with respect to its arguments:
	\begin{equation}
		\label{eq:symmetry}
		LoI(O, \hat{O}_T) = LoI(\hat{O}_T, O).
	\end{equation}
\end{proposition}

\begin{proof}
	For each pair $(i,j)$, we have $(o_{ij} - \hat{o}_{ij})^2 = (\hat{o}_{ij} - o_{ij})^2$ and $\max(o_{ij}, \hat{o}_{ij}) = \max(\hat{o}_{ij}, o_{ij})$. Hence each summand is invariant to the interchange of $O$ and $\hat{O}_T$, and so is the sum.
\end{proof}

\begin{proposition}[Upper bound]
	\label{prop:upper_bound}
	For co-occurrence matrices with entries in $[0,1]$, the LoI is bounded above:
	\begin{equation}
		\label{eq:upper_bound}
		LoI(O, \hat{O}_T) \leq M,
	\end{equation}
	where $M = n(n-1)/2$.
\end{proposition}

\begin{proof}
	Each summand satisfies:
	\begin{equation}
		\frac{(o_{ij} - \hat{o}_{ij})^2}{\max(o_{ij}, \hat{o}_{ij})} \leq \frac{\max(o_{ij}, \hat{o}_{ij})^2}{\max(o_{ij}, \hat{o}_{ij})} = \max(o_{ij}, \hat{o}_{ij}) \leq 1,
	\end{equation}
	where the first inequality follows from $|o_{ij} - \hat{o}_{ij}| \leq \max(o_{ij}, \hat{o}_{ij})$ when both values lie in $[0,1]$. Summing over all $M$ pairs yields $LoI(O, \hat{O}_T) \leq M$.
\end{proof}

\begin{proposition}[Permutation invariance]
	\label{prop:permutation_invariance}
	The LoI is invariant under simultaneous reordering of observations. For any permutation $\pi$ of $\{1, \ldots, n\}$:
	\begin{equation}
		\label{eq:permutation_invariance}
		LoI(O^\pi, \hat{O}_T^\pi) = LoI(O, \hat{O}_T),
	\end{equation}
	where $O^\pi = [o_{\pi(i)\pi(j)}]$ and $\hat{O}_T^\pi = [\hat{o}_{\pi(i)\pi(j)}]$ denote the matrices with rows and columns simultaneously permuted according to $\pi$.
\end{proposition}

\begin{proof}
	Under simultaneous permutation, the contribution of the permuted pair $(\pi(i), \pi(j))$ is
		\begin{equation}
	\frac{(o_{\pi(i)\pi(j)} - \hat{o}_{\pi(i)\pi(j)})^2}{\max(o_{\pi(i)\pi(j)}, \hat{o}_{\pi(i)\pi(j)})}
		\end{equation}
 	Since $\pi$ is a bijection, the multiset of summands is identical to the unpermuted case; only the indexing changes. As the LoI is a sum over this multiset, it remains unchanged.
\end{proof}

\begin{proposition}[Decomposability]
	\label{prop:decomposability}
	Let $\mathcal{P}_\text{in} = \{(i,j) : i < j,\, i \text{ and } j \text{ in same terminal node}\}$ and $\mathcal{P}_\text{out} = \{(i,j) : i < j,\, i \text{ and } j \text{ in different terminal nodes}\}$. Then:
	\begin{equation}
		\label{eq:decomposability}
		LoI(O, \hat{O}_T) = \underbrace{\sum_{(i,j) \in \mathcal{P}_\text{in}} \frac{(o_{ij} - \hat{o}_{ij})^2}{\max(o_{ij}, \hat{o}_{ij})}}_{LoI_\text{in}} \;+\; \underbrace{\sum_{(i,j) \in \mathcal{P}_\text{out}} o_{ij}}_{LoI_\text{out}}.
	\end{equation}
\end{proposition}

\begin{proof}
	For pairs in $\mathcal{P}_\text{out}$, we have $\hat{o}_{ij} = 0$ by the definition of the E2Tree co-occurrence matrix~\eqref{eq:reconstruction_logic}. Therefore:
	\begin{equation}
		\frac{(o_{ij} - 0)^2}{\max(o_{ij}, 0)} = \frac{o_{ij}^2}{o_{ij}} = o_{ij},
	\end{equation}
	for all pairs where $o_{ij} > 0$. When $o_{ij} = 0$, the contribution is zero by convention. For pairs in $\mathcal{P}_\text{in}$, the general formula applies directly. The total LoI is the sum of contributions from both partition classes.
\end{proof}

\begin{remark}[Diagnostic interpretation of the decomposition]
	\label{rem:decomposability}
	Proposition~\ref{prop:decomposability} provides a structural decomposition of the LoI into two interpretable components that constitutes a unique diagnostic capability of the LoI, one that no correlation-based measure (including the Mantel test) can provide.

	The within-node component $LoI_\text{in}$ captures the fidelity of the reconstruction for pairs that E2Tree groups together: it measures whether the single tree assigns co-occurrence values that match the ensemble's.

	The between-node component $LoI_\text{out}$ captures the cost of the partition: it accumulates the ensemble co-occurrence values of pairs that E2Tree separates into different terminal nodes. Since separated pairs receive $\hat{o}_{ij} = 0$, the entire ensemble co-occurrence $o_{ij}$ for those pairs becomes an irrecoverable loss.

	A critical subtlety arises when interpreting the raw totals $LoI_\text{in}$ and $LoI_\text{out}$: they are \emph{not directly comparable} because the two sets of pairs have vastly different cardinalities. For an E2Tree with $K$ terminal nodes of roughly equal size $n/K$, the number of between-node pairs scales as $M(1 - 1/K)$ while within-node pairs scale as $M/K$, so between-node pairs can outnumber within-node pairs by a factor of $K-1$. Consequently, $LoI_\text{out}$ will typically dominate the raw total even when the partition is well-placed, simply because it accumulates the ``fuzzy background'' of the ensemble proximity matrix across many more pairs.

	To obtain a meaningful diagnostic, we define the per-pair averages:
	\begin{equation}
		\label{eq:mean_in_out}
		\bar{\ell}_\text{in} = \frac{LoI_\text{in}}{|\mathcal{P}_\text{in}|}, \qquad \bar{\ell}_\text{out} = \frac{LoI_\text{out}}{|\mathcal{P}_\text{out}|}.
	\end{equation}
	These averages are directly comparable regardless of the partition structure. The quantity $\bar{\ell}_\text{in}$ represents the average calibration error for pairs that E2Tree groups together, while $\bar{\ell}_\text{out}$ represents the average ensemble proximity lost for each pair that E2Tree separates.

	This decomposition enables targeted model improvement:
	\begin{itemize}
		\item A high $\bar{\ell}_\text{out}$ (e.g., $> 0.3$) indicates that the E2Tree is separating pairs with substantial ensemble proximity. The practitioner should consider increasing the tree's complexity (more terminal nodes) or relaxing pruning constraints.
		\item A high $\bar{\ell}_\text{in}$ (e.g., $> 0.1$) indicates poor within-node calibration: the partition boundaries are well-placed, but the within-node proximity values deviate from the ensemble's. This typically reflects the inherent fuzzy-to-crisp structural transition.
		\item Low values of both $\bar{\ell}_\text{in}$ and $\bar{\ell}_\text{out}$ confirm that the E2Tree faithfully captures the ensemble structure: the partition is well-placed and the calibration is accurate.
	\end{itemize}

	Such directional diagnostics are fundamentally impossible with the Mantel test or any other scalar association measure, which can only report that the overall reconstruction is ``good'' or ``poor'' without indicating the source of the discrepancy.
\end{remark}

\subsection{Connection to the Cressie--Read Power Divergence Family}

The LoI can be formally situated within the family of power divergence statistics introduced by \citet{Cressie1984}, which has since been extended to non-Gaussian vector stationary processes \citep{OgataTaniguchi2009} and to local Whittle likelihood frameworks for generalized divergences \citep{XueTaniguchi2020}. The Cressie--Read power divergence between two discrete distributions $p$ and $q$ is defined as:
\begin{equation}
	\label{eq:cressie_read}
	I^\lambda(p, q) = \frac{1}{\lambda(\lambda + 1)} \sum_k p_k \left[ \left(\frac{p_k}{q_k}\right)^\lambda - 1 \right], \quad \lambda \in \mathbb{R},
\end{equation}
with limiting cases defined by continuity. Key members of this family include: $\lambda = 1$ (Pearson's $\chi^2$), $\lambda = 0$ (log-likelihood ratio $G^2$), $\lambda = -1/2$ (Freeman--Tukey statistic), and $\lambda = -2$ (Neyman's $\chi^2_N$).

The Neyman chi-squared statistic ($\lambda = -2$) takes the form:
\begin{equation}
	\label{eq:neyman}
	\chi^2_N = \sum_k \frac{(p_k - q_k)^2}{q_k}.
\end{equation}

Comparing equations~\eqref{eq:goi_definition} and~\eqref{eq:neyman}, we observe that the LoI uses $\max(o_{ij}, \hat{o}_{ij})$ in place of $\hat{o}_{ij}$ (or $o_{ij}$) in the denominator. This substitution represents a \emph{robustified Neyman-type statistic} with three important consequences:

\begin{enumerate}[label=(\roman*)]
	\item \textbf{Singularity avoidance.} The classical Neyman statistic is undefined when $q_k = 0$. In the E2Tree context, $\hat{o}_{ij} = 0$ is the norm rather than the exception: the block-diagonal structure of $\hat{O}_T$ forces all between-node pairs to have zero co-occurrence. Using $\max(o_{ij}, \hat{o}_{ij})$ ensures a well-defined denominator whenever at least one matrix assigns positive proximity.

	\item \textbf{Symmetry.} Unlike the classical Neyman statistic, which treats $p$ and $q$ asymmetrically, the LoI is symmetric in $O$ and $\hat{O}_T$ (demonstrated earlier in Proposition~\ref{prop:symmetry}). This property is desirable because neither matrix has a privileged role as the ``expected'' distribution.

	\item \textbf{Weighting interpretation.} Each pair's contribution is normalized by the larger of the two co-occurrence values, which ensures that discrepancies are measured relative to the ``importance'' of the pair. This preserves the Neyman weighting intuition (errors normalized by magnitude) while avoiding the extreme sensitivity to small denominators that plagues the classical formulation.
\end{enumerate}

\begin{proposition}[Relationship to Neyman $\chi^2$]
	\label{prop:cressie_read}
	For all pairs $(i,j)$ with $i < j$:
	\begin{equation}
		\frac{(o_{ij} - \hat{o}_{ij})^2}{\max(o_{ij}, \hat{o}_{ij})} \leq \frac{(o_{ij} - \hat{o}_{ij})^2}{\min(o_{ij}, \hat{o}_{ij})}
	\end{equation}
	whenever $\min(o_{ij}, \hat{o}_{ij}) > 0$. Equality holds if and only if $o_{ij} = \hat{o}_{ij}$.
\end{proposition}

\begin{proof}
	Since $\max(a,b) \geq \min(a,b)$ for all $a,b \geq 0$, the result follows from the monotone decreasing property of $f(x) = c/x$ for $c \geq 0$ and $x > 0$.
\end{proof}

This proposition shows that the LoI provides a \emph{conservative} (lower) estimate compared to a Neyman-type statistic using $\min(o_{ij}, \hat{o}_{ij})$ in the denominator, which would be unstable when either entry is near zero.

The divergence-based perspective also connects to the theory of proper scoring rules \citep{Dawid2016}, in which objective functions of divergence type are used as alternatives to the log-likelihood. In that framework, each member of the Cressie--Read family corresponds to a specific scoring rule, providing an additional theoretical underpinning for the $\lambda = -2$ specialization adopted here.

\section{A Family of Alternative Divergence Measures}
\label{sec:alternatives}

No single measure captures every facet of proximity matrix agreement. We introduce four complementary statistics, each with distinct theoretical properties, that address aspects of reconstruction quality the nLoI does not directly target.

\subsection{Hellinger Distance}
\label{sec:hellinger}

The Hellinger distance \citep{Hellinger1909} between the lower-triangular elements of $O$ and $\hat{O}_T$ is defined as:
\begin{equation}
	\label{eq:hellinger}
	H(O, \hat{O}_T) = \sqrt{\frac{1}{M} \sum_{i < j} \left(\sqrt{o_{ij}} - \sqrt{\hat{o}_{ij}}\right)^2}.
\end{equation}

The square-root transformation $\sqrt{\cdot}$ that defines the Hellinger distance has deep connections to variance stabilization \citep{Freeman1950}. For count data following a Poisson distribution with mean $\mu$, the variance-stabilizing transformation is approximately $2\sqrt{X}$, yielding $Var(2\sqrt{X}) \approx 1$ independently of $\mu$. This property extends to our setting: the square-root transformation equalizes the contribution of each pair regardless of its magnitude, preventing high-proximity pairs from dominating the measure.

\begin{proposition}[Properties of $H$]
	\label{prop:hellinger_properties}
	The Hellinger distance satisfies:
	\begin{enumerate}[label=(\roman*)]
		\item $H(O, \hat{O}_T) \geq 0$, with equality if and only if $o_{ij} = \hat{o}_{ij}$ for all $i < j$.
		\item $H(O, \hat{O}_T) = H(\hat{O}_T, O)$ (symmetry).
		\item $0 \leq H(O, \hat{O}_T) \leq 1$ when $o_{ij}, \hat{o}_{ij} \in [0,1]$.
		\item $H$ satisfies the triangle inequality and is therefore a metric on the space of proximity matrices.
	\end{enumerate}
\end{proposition}

\begin{proof}
	Properties (i)--(iii) follow directly from the definition. For (iii), note that $(\sqrt{a} - \sqrt{b})^2 \leq \max(a,b) \leq 1$ for $a,b \in [0,1]$, so $H^2 \leq 1$. Property (iv) follows from the fact that $H$ is the $L^2$ distance between $(\sqrt{o_{ij}})_{i<j}/\sqrt{M}$ and $(\sqrt{\hat{o}_{ij}})_{i<j}/\sqrt{M}$.
\end{proof}

  The Hellinger distance is connected to the Freeman-Tukey statistic \citep{Freeman1950}: $\mathrm{FT} = 4M \cdot H^2$,  corresponding to the Cressie-Read family at $\lambda = -1/2$.
The properties and applications of this statistic in multivariate settings, including its role in correspondence analysis, are examined in \citet{Beh2018}.

A key practical advantage of the Hellinger distance is its robustness to sparsity. Since $\sqrt{0} = 0$ is well-defined, no special handling of zero entries is required. Moreover, the absence of a denominator that can approach zero eliminates the numerical instabilities that affect ratio-based measures like nLoI when proximity values are small. This makes the Hellinger distance particularly suitable for evaluating E2Tree reconstructions with many terminal nodes, where $\hat{O}_T$ is highly sparse.

\subsection{Weighted Root Mean Squared Error}
\label{sec:wrmse}

The weighted RMSE shares the same weighting philosophy as the nLoI but expresses the result in a more interpretable RMSE scale:
\begin{equation}
	\label{eq:wrmse}
	\mathrm{wRMSE}(O, \hat{O}_T) = \sqrt{\frac{\sum_{i<j} w_{ij} \cdot (o_{ij} - \hat{o}_{ij})^2}{\sum_{i<j} w_{ij}}},
\end{equation}
where $w_{ij} = \max(o_{ij}, \hat{o}_{ij}, \varepsilon)$ and $\varepsilon > 0$ is a small constant for numerical stability.

\begin{proposition}[Properties of wRMSE]
	\label{prop:wrmse_properties}
	The wRMSE satisfies:
	\begin{enumerate}[label=(\roman*)]
		\item $\mathrm{wRMSE} \geq 0$, with equality if and only if $o_{ij} = \hat{o}_{ij}$ for all $i < j$.
		\item $0 \leq \mathrm{wRMSE} \leq 1$ when $o_{ij}, \hat{o}_{ij} \in [0,1]$.
		\item $\mathrm{wRMSE}$ is symmetric in $O$ and $\hat{O}_T$.
	\end{enumerate}
\end{proposition}

\begin{proof}
	(i) and (iii) are immediate. For (ii), note that $|o_{ij} - \hat{o}_{ij}| \leq 1$ implies $(o_{ij} - \hat{o}_{ij})^2 \leq 1$, and $w_{ij} \leq 1$ with $\sum w_{ij} > 0$, so the weighted average is at most 1.
\end{proof}

The wRMSE can be interpreted as a ``weighted average discrepancy'' where the weights emphasize high-proximity regions. Unlike the nLoI, which divides each squared discrepancy by the weight, the wRMSE multiplies by the weight. This distinction means that the wRMSE treats the weights as importance indicators (high-proximity pairs contribute more to the average) rather than as normalizers (high-proximity pairs have their discrepancies attenuated).

\subsection{RV Coefficient}
\label{sec:rv}

The RV coefficient \citep{Robert1976} provides a matrix-level analog of Pearson's correlation coefficient:
\begin{equation}
	\label{eq:rv}
	\mathrm{RV}(O, \hat{O}_T) = \frac{\tr(O_c \cdot \hat{O}_c)}{\sqrt{\tr(O_c^2) \cdot \tr(\hat{O}_c^2)}},
\end{equation}
where $O_c$ and $\hat{O}_c$ denote the matrices with diagonals set to zero (since $o_{ii} = \hat{o}_{ii} = 1$ for all $i$, the diagonal carries no information about reconstruction quality).

The trace inner product $\tr(A \cdot B) = \sum_{ij} a_{ij} b_{ij}$ is the Frobenius inner product, and the RV coefficient is the cosine of the angle between $O_c$ and $\hat{O}_c$ in Frobenius space.

\begin{proposition}[Properties of RV]
	\label{prop:rv_properties}
	For symmetric non-negative matrices:
	\begin{enumerate}[label=(\roman*)]
		\item $0 \leq \mathrm{RV}(O, \hat{O}_T) \leq 1$.
		\item $\mathrm{RV} = 1$ if and only if $O_c = c \cdot \hat{O}_c$ for some scalar $c > 0$.
		\item $\mathrm{RV} = 0$ if and only if $\tr(O_c \cdot \hat{O}_c) = 0$, i.e., the matrices are orthogonal in the Frobenius inner product.
		\item $\mathrm{RV}$ is invariant to positive scaling of either matrix.
	\end{enumerate}
\end{proposition}

\begin{proof}
	These follow from the Cauchy--Schwarz inequality applied to the Frobenius inner product space. Non-negativity of both matrices ensures non-negative trace products.
\end{proof}

Unlike the element-wise divergences (nLoI, Hellinger, wRMSE), the RV coefficient captures \emph{global structural similarity} without requiring element-wise correspondence. This is valuable because E2Tree may approximate the overall pattern of $O$ even when individual entries differ. However, the scale invariance (property iv) means that the RV coefficient cannot distinguish between a perfect reconstruction and one that preserves structure but at a different scale. The significance of the RV coefficient can be assessed through permutation testing \citep{Josse2008}, through the approximation to a scaled $\chi^2$ distribution or via asymptotic tests valid in high dimensions \citep{Ahmad2019}..

\subsection{Structural Similarity Index (SSIM)}
\label{sec:ssim}

The SSIM \citep{Wang2004}, originally developed for image quality assessment, evaluates local structural patterns by computing similarity within a sliding window:
\begin{equation}
	\label{eq:ssim}
	\mathrm{SSIM}(O, \hat{O}_T) = \frac{1}{|\mathcal{W}|} \sum_{w \in \mathcal{W}} \frac{(2\mu_{O,w}\mu_{\hat{O},w} + C_1)(2\sigma_{O\hat{O},w} + C_2)}{(\mu_{O,w}^2 + \mu_{\hat{O},w}^2 + C_1)(\sigma_{O,w}^2 + \sigma_{\hat{O},w}^2 + C_2)},
\end{equation}
where $\mathcal{W}$ is the set of window positions, $\mu_{\cdot,w}$ and $\sigma_{\cdot,w}^2$ denote local means and variances within window $w$, $\sigma_{O\hat{O},w}$ is the local covariance, and $C_1 = (0.01L)^2$, $C_2 = (0.03L)^2$ are stabilization constants with $L = 1$ (the dynamic range of $[0,1]$ matrices).

The application of SSIM to proximity matrices is natural: both $O$ and $\hat{O}_T$ can be viewed as grayscale images, and the block-diagonal structure of $\hat{O}_T$ represents a spatial pattern that SSIM is specifically designed to evaluate. The three components of the SSIM formula capture luminance similarity (local means), contrast similarity (local variances), and structural correlation (local covariance), providing a richer comparison than purely element-wise measures.

\begin{proposition}[Properties of SSIM]
	\label{prop:ssim_properties}
	The SSIM satisfies:
	\begin{enumerate}[label=(\roman*)]
		\item $-1 \leq \mathrm{SSIM} \leq 1$, with $\mathrm{SSIM} = 1$ if and only if $O = \hat{O}_T$ within every window.
		\item $\mathrm{SSIM}$ is symmetric in $O$ and $\hat{O}_T$.
		\item $\mathrm{SSIM}$ can be computed in $O(n^2)$ time using cumulative-sum convolution.
	\end{enumerate}
\end{proposition}

\subsection{Summary and Comparison of Measures}

Table~\ref{tab:measures_comparison} summarizes the theoretical properties of all measures, including the Mantel test as a baseline.

\begin{table}[h!]
	\centering
	\caption{Comparison of divergence and similarity measures for proximity matrix comparison.}
	\label{tab:measures_comparison}
	\small
	\begin{tabular}{lcccccc}
		\toprule
		Property & nLoI & Hellinger & wRMSE & RV & SSIM & Mantel \\
		\midrule
		Type & divergence & divergence & divergence & similarity & similarity & correlation \\
		Range & $[0, 1]$ & $[0, 1]$ & $[0, 1]$ & $[0, 1]$ & $[-1, 1]$ & $[-1, 1]$ \\
		Scale-sensitive & Yes & Yes & Yes & No & Yes & No \\
		Sparse-robust & Moderate & Excellent & Good & Moderate & Good & Poor \\
		Element-wise & Yes & Yes & Yes & No & No & Yes \\
		Metric & Yes & Yes & No$^a$ & No & No & No \\
		Zero at identity & Yes & Yes & Yes & No$^b$ & No$^b$ & No$^b$ \\
		\bottomrule
		\multicolumn{7}{l}{\footnotesize $^a$ wRMSE does not satisfy the triangle inequality in general.} \\
		\multicolumn{7}{l}{\footnotesize $^b$ Similarity measures equal 1 (not 0) at perfect agreement.}
	\end{tabular}
\end{table}

The measures form a complementary family. The nLoI, Hellinger, and wRMSE are \emph{divergence measures} (lower values indicate better reconstruction), while the RV coefficient and SSIM are \emph{similarity measures} (higher values indicate better reconstruction). The divergence measures are element-wise (they sum over pair-level contributions), while the RV captures global structure and SSIM captures local spatial patterns. The Hellinger distance is the only measure that simultaneously satisfies all metric axioms, has a bounded range, and exhibits robust behavior under sparsity.

\section{Permutation Testing Framework}
\label{sec:permutation}

We develop a unified permutation testing framework that applies to all measures in the proposed family. A critical component of this framework is the choice of permutation scheme, which determines the null hypothesis being tested.

\subsection{Hypotheses}

For divergence measures (nLoI, Hellinger, wRMSE):
\begin{itemize}
	\item $H_0$: The E2Tree structure is not associated with the ensemble structure.
	\item $H_1$: The observed divergence is significantly \emph{lower} than expected under random label correspondence.
\end{itemize}

For similarity measures (RV, SSIM):
\begin{itemize}
	\item $H_0$: The E2Tree structure is not associated with the ensemble structure.
	\item $H_1$: The observed similarity is significantly \emph{higher} than expected under random label correspondence.
\end{itemize}

\subsection{The Row/Column Permutation Scheme}
\label{sec:perm_scheme}

The null distribution is generated by applying \emph{simultaneous row/column permutations} to $\hat{O}_T$. For each permutation replicate, a random permutation $\pi$ of $\{1, \ldots, n\}$ is drawn, and the permuted matrix $\hat{O}_T^\pi = [\hat{o}_{\pi(i)\pi(j)}]$ is constructed by applying $\pi$ to both rows and columns simultaneously.

\begin{algorithm}[h!]
	\caption{Unified Permutation Test for E2Tree Validation}
	\label{alg:permutation}
	\begin{algorithmic}[1]
		\Require $O$, $\hat{O}_T$; divergence/similarity functions $d_1, \ldots, d_K$; permutations $R$; level $\alpha$
		\Ensure $p$-values, confidence intervals, $Z$-statistics for each measure
		\State Compute observed statistics: $\tau_k^{\text{obs}} \gets d_k(O, \hat{O}_T)$ for $k = 1, \ldots, K$
		\For{$r = 1, \ldots, R$}
		\State $\pi \gets$ random permutation of $\{1, \ldots, n\}$
		\State $\hat{O}_T^\pi \gets \hat{O}_T[\pi, \pi]$ \Comment{Simultaneous row/column permutation}
		\For{$k = 1, \ldots, K$}
		\State $\tau_k^{(r)} \gets d_k(O, \hat{O}_T^\pi)$
		\EndFor
		\EndFor
		\For{$k = 1, \ldots, K$}
		\State $\hat{\mu}_{0,k} \gets R^{-1} \sum_{r} \tau_k^{(r)}$; \quad $\hat{\sigma}_{0,k} \gets \text{sd}(\tau_k^{(1)}, \ldots, \tau_k^{(R)})$
		\If{$d_k$ is a divergence}
		\State $p_k \gets (1 + \sum_{r} \{\tau_k^{(r)} \leq \tau_k^{\text{obs}}\}) / (1 + R)$
		\Else
		\State $p_k \gets (1 + \sum_{r} \{\tau_k^{(r)} \geq \tau_k^{\text{obs}}\}) / (1 + R)$
		\EndIf
		\State $Z_k \gets (\tau_k^{\text{obs}} - \hat{\mu}_{0,k}) / \hat{\sigma}_{0,k}$
		\State Compute CI as in equations~\eqref{eq:perm_ci_raw}--\eqref{eq:perm_ci_bounded}
		\EndFor
	\end{algorithmic}
\end{algorithm}

This procedure has three key advantages over the value-shuffle scheme:

\begin{enumerate}[label=(\roman*)]
	\item \textbf{Preservation of internal structure.} Row/column permutation preserves the block-diagonal structure of $\hat{O}_T$: the permuted matrix $\hat{O}_T^\pi$ has the same block sizes, the same within-block values, and the same overall sparsity pattern as $\hat{O}_T$. Only the correspondence between observation labels and blocks is randomized.

	\item \textbf{Correct null hypothesis.} The row/column permutation tests whether the \emph{labeling} of observations in $\hat{O}_T$ corresponds to the labeling in $O$. This is the appropriate null for E2Tree validation: under $H_0$, the E2Tree groups observations randomly with respect to the ensemble's proximity structure.

	\item \textbf{Consistency with established methodology.} The row/column permutation scheme is identical to the one used in the Mantel test \citep{Mantel1967} and PROTEST \citep{Peres-Neto2001}, ensuring methodological coherence when comparing our measures against the Mantel baseline.
\end{enumerate}

In contrast, value-shuffle permutation (randomly reassigning individual $\hat{o}_{ij}$ values across positions) destroys the block structure of $\hat{O}_T$, generating null matrices that do not resemble any plausible E2Tree output. This tests a different as well as a weaker null hypothesis and can produce misleading null distributions.

\subsection{Efficiency of the Unified Framework}
The efficiency of permutation testing in ensemble contexts has been studied by \citet{Hapfelmeier2023}, who develop strategies for reducing the resampling cost when testing variable importance measures in random forests. Our unified framework achieves analogous efficiency gains by reusing a single set of $R$ permutations across all $K$ measures (Algorithm~\ref{alg:permutation}, lines 2--7). This ensures that: (i) all measures are evaluated on identical permuted matrices, enabling direct comparison of $p$-values; (ii) the total computational cost is $O(R \cdot K \cdot f(n))$ where $f(n)$ is the cost of computing a single measure, rather than $O(K \cdot R \cdot f(n))$ with separate permutation loops; and (iii) the correlation structure among null distributions is preserved, facilitating joint interpretation.

\subsection{Confidence Intervals}

The permutation distribution provides a basis for constructing confidence intervals. Let $Q_{\alpha/2}$ and $Q_{1-\alpha/2}$ denote the $\alpha/2$ and $1-\alpha/2$ quantiles of the permutation distribution $\{\tau_1, \ldots, \tau_R\}$. The $(1-\alpha)$ confidence interval is:
\begin{equation}
	\label{eq:perm_ci_raw}
	\text{CI}_L = \tau_{\text{obs}} + \hat{\mu}_0 - Q_{1-\alpha/2}, \qquad \text{CI}_U = \tau_{\text{obs}} + \hat{\mu}_0 - Q_{\alpha/2}.
\end{equation}

For non-negative divergence measures, the lower bound is truncated:
\begin{equation}
	\label{eq:perm_ci_bounded}
	\text{CI}_{1-\alpha}^{\text{perm}} = \left[\max(0, \text{CI}_L), \; \text{CI}_U\right].
\end{equation}

For similarity measures bounded in $[0,1]$ (or $[-1,1]$ for SSIM), both bounds are appropriately clamped.

The $p$-value computation includes the Phipson--Smyth correction \citep{phipson2010permutation}, adding 1 to both numerator and denominator to ensure $p \in [1/(R+1),\, 1]$ and prevent the theoretically impossible value of $p = 0$. When the number of permutations $R$ is large but the cost of each evaluation is non-trivial, sequential Monte Carlo approaches such as those of \citet{GandyHahn2014} can guarantee that test decisions match those that would be obtained with exact $p$-values, thereby controlling the \emph{resampling risk} that an approximate $p$-value may fall on the wrong side of the significance threshold \citep{GandyHahnDing2020}.

\subsection{Standardized Test Statistic}

The $Z$-score:
\begin{equation}
	\label{eq:zscore}
	Z = \frac{\tau_{\text{obs}} - \hat{\mu}_0}{\hat{\sigma}_0},
\end{equation}
provides a scale-invariant quantification of the departure from the null expectation. For divergence measures, large negative $Z$ indicates strong reconstruction fidelity; for similarity measures, large positive $Z$ serves the same role.

\section{Monte Carlo Simulation Study}
\label{sec:simulation}

We conduct three simulation experiments to evaluate: (i) Type~I error control under the null hypothesis, (ii) statistical power as a function of signal strength, and (iii) robustness to matrix sparsity. The design follows the general structure of extended Monte Carlo simulation studies for permutation-based inference \citep{Ditzhaus2025}, where Type~I error and power are assessed jointly across a factorial grid of configurations.

\subsection{Data Generation}
\label{sec:sim_design}

For each simulation replicate, we generate a pair of matrices $(O, \hat{O})$ with controlled properties:

\begin{itemize}
	\item \textbf{Ensemble matrix $O$} (fuzzy): A symmetric $n \times n$ matrix with block structure reflecting $K$ groups. Within-group entries are drawn from $\mathcal{N}(0.7, 0.15^2)$ and between-group entries from $\mathcal{N}(0.1, 0.15^2)$, both truncated to $[0,1]$.

	\item \textbf{E2Tree matrix $\hat{O}$} (crisp): A block-diagonal matrix with binary structure ($\hat{o}_{ij} = 1$ within blocks, $\hat{o}_{ij} = 0$ between blocks), mimicking E2Tree's partition structure.

	\item \textbf{Signal parameter $s \in [0,1]$}: Controls the correspondence between $\hat{O}$'s block structure and $O$'s group structure. At $s = 0$, $\hat{O}$ uses a random permutation of group labels (H$_0$); at $s = 1$, $\hat{O}$ perfectly matches $O$'s groups. Intermediate values reassign a fraction $1 - s$ of observations to random groups.

	\item \textbf{Sparsity parameter}: Controls the proportion of between-group entries in $O$ forced to zero, simulating increased sparsity.
\end{itemize}

\subsection{Simulation 1: Type I Error Control}
\label{sec:sim_size}

\subsubsection{Design}

We evaluate whether each measure maintains the nominal Type~I error rate under $H_0$ (signal $= 0$). The design crosses three sample sizes ($n \in \{50, 100, 200\}$) with three group counts ($K \in \{3, 5, 10\}$), yielding 9 configurations. For each configuration, we generate 200 replicates, each tested with 999 permutations at $\alpha = 0.05$. The measures evaluated are nLoI, Hellinger, wRMSE, RV, and Mantel.

\subsubsection{Results}

Table~\ref{tab:type1} reports the empirical rejection rates. Under correct Type~I error control, rejection rates should fall within the 95\% binomial confidence interval $[0.025, 0.070]$ for 200 replicates at $\alpha = 0.05$.

\begin{table}[h!]
	\centering
	\caption{Empirical rejection rates under $H_0$ (signal $= 0$). Nominal level $\alpha = 0.05$, 200 replicates, 999 permutations per replicate. All rates fall within the acceptable range $[0.025, 0.070]$.}
	\label{tab:type1}
	\small
	\begin{tabular}{rr|ccccc}
		\toprule
		$n$ & $K$ & nLoI & Hellinger & wRMSE & RV & Mantel \\
		\midrule
		50  & 3  & 0.040 & 0.045 & 0.045 & 0.045 & 0.045 \\
		50  & 5  & 0.055 & 0.045 & 0.045 & 0.045 & 0.055 \\
		50  & 10 & 0.045 & 0.050 & 0.045 & 0.040 & 0.030 \\
		100 & 3  & 0.065 & 0.065 & 0.065 & 0.065 & 0.060 \\
		100 & 5  & 0.055 & 0.060 & 0.050 & 0.050 & 0.055 \\
		100 & 10 & 0.030 & 0.035 & 0.030 & 0.035 & 0.050 \\
		200 & 3  & 0.030 & 0.025 & 0.025 & 0.030 & 0.030 \\
		200 & 5  & 0.030 & 0.040 & 0.050 & 0.050 & 0.045 \\
		200 & 10 & 0.055 & 0.060 & 0.070 & 0.055 & 0.050 \\
		\bottomrule
	\end{tabular}
\end{table}

All measures maintain rejection rates within the acceptable bounds across all 9 configurations. No systematic inflation or deflation is observed, confirming that the row/column permutation scheme provides valid inference for all measures in the family.

\subsubsection{Permutation Scheme Comparison}

To verify that the choice between row/column and value-shuffle permutation does not affect Type~I error control, we conducted an auxiliary experiment comparing both schemes for nLoI and Hellinger across 100 replicates at $n = 100$, $K = 5$, signal $= 0$. Both schemes produced rejection rates consistent with the nominal $\alpha = 0.05$ level. Under $H_0$, where no structure links $O$ to $\hat{O}_T$, both permutation schemes yield valid null distributions. The distinction becomes important under $H_1$, where value-shuffle can inflate Type~I error by generating implausible null matrices.

\subsection{Simulation 2: Statistical Power}
\label{sec:sim_power}

\subsubsection{Design}

We evaluate the power of each measure to detect genuine E2Tree reconstruction across signal strengths $s \in \{0, 0.1, 0.2, \ldots, 1.0\}$, sample sizes $n \in \{50, 100, 200\}$, with $K = 5$ groups, sparsity $= 0.3$. Each of the 33 configurations is replicated 200 times with 999 permutations. The measures evaluated are nLoI, Hellinger, wRMSE, and RV.

\subsubsection{Results}

Figure~\ref{fig:power_curves} displays the power curves, and Table~\ref{tab:power_key} reports rejection rates at selected signal levels.

\begin{figure}[h!]
	\centering
	\includegraphics[width=\textwidth]{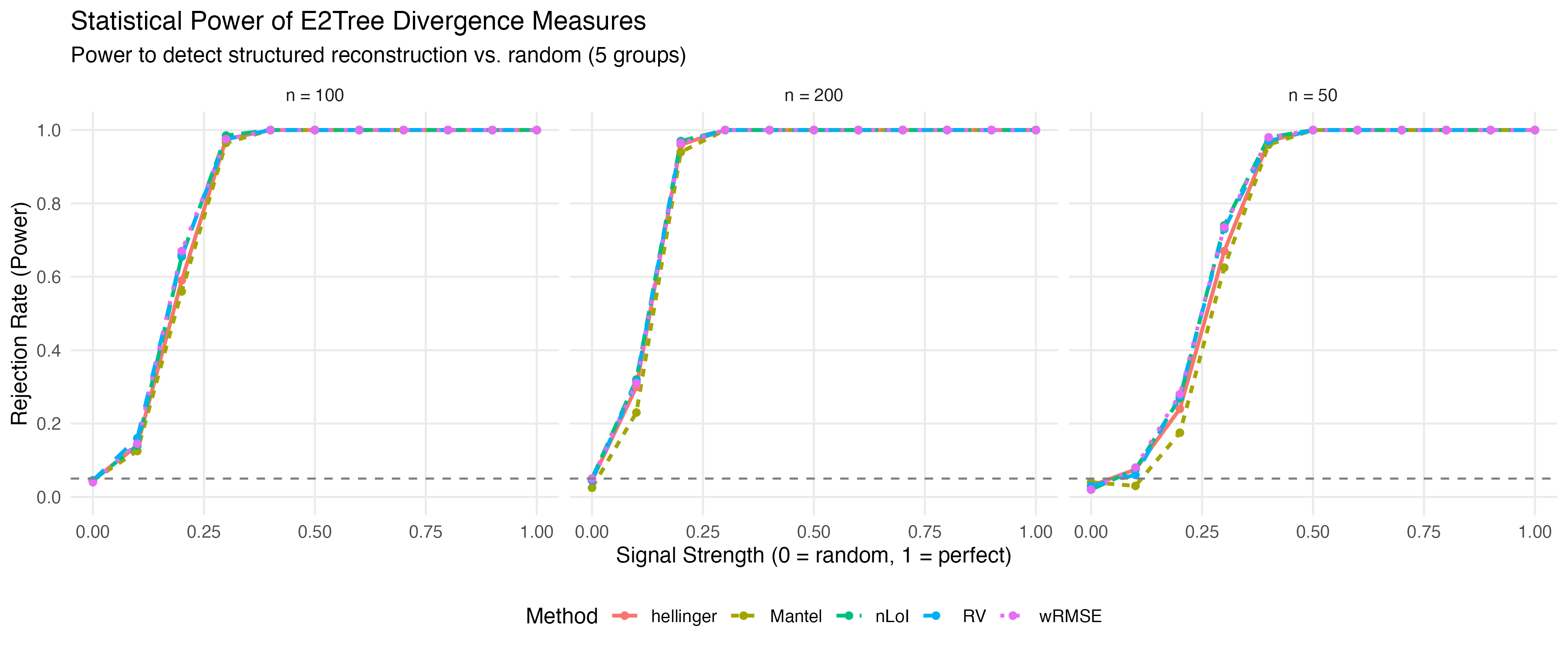}
	\caption{Statistical power as a function of signal strength for five measures (nLoI, Hellinger, wRMSE, RV, Mantel), faceted by sample size $n$. Dashed horizontal line indicates the nominal level $\alpha = 0.05$. Five groups, sparsity $= 0.3$, 200 replicates per configuration.}
	\label{fig:power_curves}
\end{figure}

\begin{table}[h!]
	\centering
	\caption{Statistical power at selected signal levels ($K = 5$, sparsity $= 0.3$, 200 replicates).}
	\label{tab:power_key}
	\small
	\begin{tabular}{rr|ccccc}
		\toprule
		$n$ & Signal & nLoI & Hellinger & wRMSE & RV & Mantel \\
		\midrule
		50  & 0.0 & 0.020 & 0.030 & 0.020 & 0.030 & 0.040 \\
		50  & 0.1 & 0.075 & 0.075 & 0.080 & 0.060 & 0.030 \\
		50  & 0.2 & 0.270 & 0.240 & 0.280 & 0.275 & 0.175 \\
		50  & 0.3 & 0.740 & 0.670 & 0.730 & 0.730 & 0.625 \\
		50  & 0.4 & 0.980 & 0.970 & 0.980 & 0.970 & 0.960 \\
		\midrule
		100 & 0.0 & 0.045 & 0.045 & 0.040 & 0.045 & 0.045 \\
		100 & 0.1 & 0.140 & 0.140 & 0.140 & 0.160 & 0.125 \\
		100 & 0.2 & 0.655 & 0.590 & 0.670 & 0.660 & 0.560 \\
		100 & 0.3 & 0.985 & 0.975 & 0.970 & 0.975 & 0.965 \\
		\midrule
		200 & 0.0 & 0.050 & 0.045 & 0.050 & 0.045 & 0.025 \\
		200 & 0.1 & 0.320 & 0.300 & 0.310 & 0.315 & 0.230 \\
		200 & 0.2 & 0.970 & 0.960 & 0.960 & 0.965 & 0.940 \\
		200 & 0.3 & 1.000 & 1.000 & 1.000 & 1.000 & 1.000 \\
		\bottomrule
	\end{tabular}
\end{table}

Two conclusions stand out. All measures achieve adequate power, reaching $\geq 0.95$ by signal $= 0.4$ at $n = 50$, and by signal $= 0.3$ for $n \geq 100$; for $s \geq 0.5$, power is uniformly 1.0. The permutation framework thus provides sufficient sensitivity across the full range of practically relevant signal strengths.

The measures also show essentially equivalent power. Differences at low signal ($s \in [0.1, 0.3]$) are small: at $n = 50$, $s = 0.3$, power ranges from 0.625 (Mantel) to 0.740 (nLoI). This near-equivalence is not a weakness of the proposal: it confirms that the choice among measures should be guided by the \emph{question being asked}, not by differential sensitivity. The Mantel test does not lack power; it measures the wrong quantity for reconstruction validation.

\subsection{Simulation 3: Robustness to Sparsity}
\label{sec:sim_sparsity}

\subsubsection{Design}

We evaluate how power degrades as the matrices become sparser. The design crosses two sample sizes ($n \in \{100, 200\}$) with three group counts ($K \in \{5, 10, 20\}$) and five sparsity levels ($\{0, 0.2, 0.4, 0.6, 0.8\}$), all at moderate signal ($s = 0.6$). Since $\hat{O}_T$ has inherent sparsity of approximately $1 - 1/K$ (the fraction of zero off-diagonal entries), increasing $K$ from 5 to 20 raises the effective sparsity from 0.80 to 0.95.

\subsubsection{Results}

All measures achieved power $= 1.0$ across all 30 configurations, including the most extreme case ($K = 20$ groups, additional sparsity $= 0.8$, effective sparsity $\approx 0.95$). At the chosen signal level ($s = 0.6$) and sample sizes ($n \geq 100$), no measure showed any degradation due to sparsity.

The result confirms robustness to the structural sparsity inherent in E2Tree matrices: even when 95\% of the entries in $\hat{O}_T$ are zero, the permutation test correctly detects the reconstruction signal. The saturation at $s = 0.6$ also indicates that differential sparsity effects are most likely to surface at weaker signals ($s \in [0.1, 0.3]$), where the power curves in Section~\ref{sec:sim_power} already exhibit measure-level differentiation. Finer exploration of the sparsity~$\times$~signal interaction, and in particular the regime where the Hellinger distance's variance-stabilising transformation confers a practical edge over the nLoI, remains a natural direction for future work.

\section{Empirical Evaluation}
\label{sec:empirical}

We evaluate the proposed measures on three benchmark datasets spanning classification and regression tasks. For each dataset, a Random Forest ensemble is trained, the ensemble co-occurrence matrix $O$ is computed, and an E2Tree is constructed to produce $\hat{O}_T$. All five measures (nLoI, Hellinger, wRMSE, RV, SSIM) are then computed with 999 row/column permutations. The datasets used in this study (Iris, mtcars, and Boston Housing) are standard benchmarks and are openly accessible through the Kaggle repository.

\subsection{Iris (Classification)}

The Iris dataset consists of 150 observations across three species with four morphological features. Following standard practice, 70\% of the data ($n = 105$) is used for training. A Random Forest with 500 trees is fitted, and the E2Tree is constructed from the resulting co-occurrence matrix.

Figure~\ref{fig:matrices} illustrates the structural difference between the two co-occurrence matrices. The Random Forest matrix $O$ (left panel) exhibits a ``fuzzy'' appearance with continuous values reflecting varying degrees of co-occurrence across the ensemble's trees. In contrast, the E2Tree matrix $\hat{O}_T$ (right panel) displays the characteristic block structure: observations within the same terminal node show high co-occurrence (dark), while observations in different terminal nodes have zero co-occurrence (light). Despite this fundamental structural constraint, the E2Tree successfully recovers the main clustering pattern of the ensemble.

\begin{figure}[h!]
	\centering
	\subfloat[Random Forest co-occurrence matrix $O$]{%
		\includegraphics[width=0.45\textwidth]{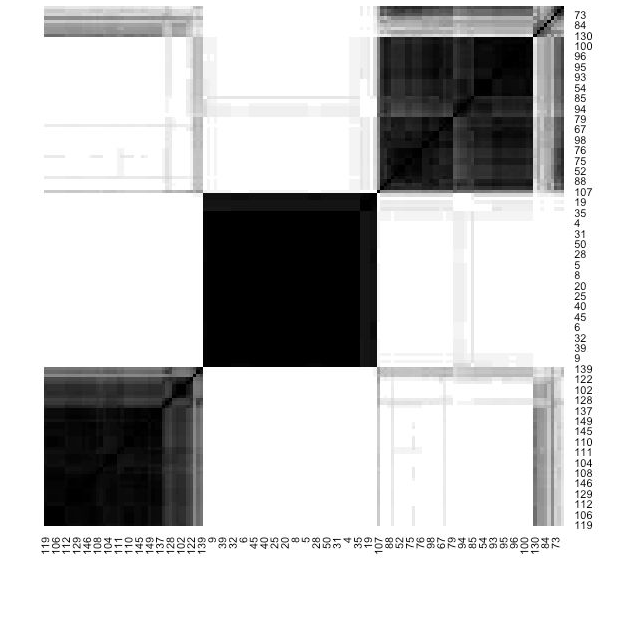}}
	\hfill
	\subfloat[E2Tree co-occurrence matrix $\hat{O}_T$]{%
		\includegraphics[width=0.45\textwidth]{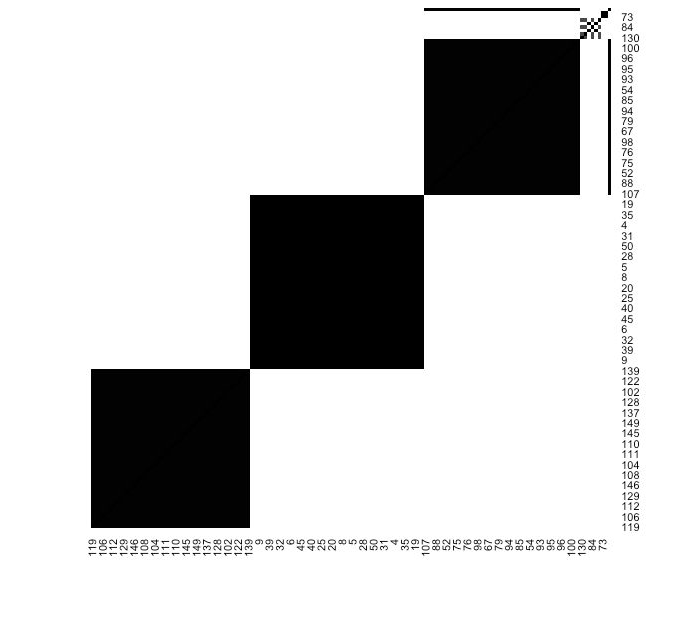}}
	\caption{Comparison of co-occurrence matrices for the Iris dataset.}
	\label{fig:matrices}
\end{figure}

Table~\ref{tab:iris_results} presents the results for all five measures.

\begin{table}[h!]
	\centering
	\caption{Multi-measure validation results for the Iris dataset ($n = 105$, $R = 999$).}
	\label{tab:iris_results}
	\small
	\begin{tabular}{llrrrrr}
		\toprule
		Method & Type & Observed & Null mean & Null SD & $Z$ & $p$ \\
		\midrule
		nLoI      & div & 0.049 & 0.406 & 0.005 & $-72.3$  & $< 0.001$ \\
		Hellinger & div & 0.211 & 0.624 & 0.004 & $-104.1$ & $< 0.001$ \\
		wRMSE     & div & 0.206 & 0.831 & 0.004 & $-171.4$ & $< 0.001$ \\
		RV        & sim & 0.960 & 0.327 & 0.009 & $+72.1$  & $< 0.001$ \\
		SSIM      & sim & 0.582 & 0.009 & 0.005 & $+110.2$ & $< 0.001$ \\
		\bottomrule
	\end{tabular}
\end{table}

All measures provide overwhelming evidence of significant reconstruction ($p < 0.001$ for all). The nLoI of 0.049 indicates extremely low per-pair divergence. The Hellinger distance of 0.211 (on a $[0,1]$ scale) indicates good but not perfect reconstruction, consistent with the visual comparison in Figure~\ref{fig:matrices}. The RV coefficient of 0.960 confirms near-perfect global structural agreement.

\subsection{mtcars (Regression)}

The mtcars dataset contains 32 observations with 10 predictors for fuel consumption prediction. A 70\% training split yields $n = 22$. Despite the small sample size, this dataset provides a useful test of the measures' behavior in a regression setting. A Random Forest with 500 trees is fitted.

\begin{table}[h!]
	\centering
	\caption{Multi-measure validation results for the mtcars dataset ($n = 22$, $R = 999$).}
	\label{tab:auto_results}
	\small
	\begin{tabular}{llrrrrr}
		\toprule
		Method & Type & Observed & Null mean & Null SD & $Z$ & $p$ \\
		\midrule
		nLoI      & div & 0.144 & 0.338 & 0.015 & $-13.0$  & $< 0.001$ \\
		Hellinger & div & 0.344 & 0.543 & 0.014 & $-14.6$ & $< 0.001$ \\
		wRMSE     & div & 0.384 & 0.700 & 0.016 & $-20.4$ & $< 0.001$ \\
		RV        & sim & 0.845 & 0.338 & 0.039 & $+13.0$  & $< 0.001$ \\
		SSIM      & sim & 0.500 & 0.138 & 0.037 & $+9.9$  & $< 0.001$ \\
		\bottomrule
	\end{tabular}
\end{table}

The regression setting presents a more challenging reconstruction task. The nLoI of 0.144 and Hellinger distance of 0.344 indicate larger reconstruction discrepancy than the Iris case, consistent with the greater complexity of regression proximity structures and the small sample size. Nonetheless, all measures achieve $p < 0.001$, confirming significant reconstruction even with only 22 observations. The RV coefficient of 0.845 indicates strong global structural agreement. The SSIM of 0.500, substantially higher than might be expected from a small matrix, suggests that the block structure is well-preserved.

\subsection{Boston Housing (Regression)}

The Boston Housing dataset contains 506 observations with 13 predictors. A 70\% training split yields $n = 354$. The larger matrix size and more complex regression structure provide a stringent test.

\begin{table}[h!]
	\centering
	\caption{Multi-measure validation results for the Boston Housing dataset ($n = 354$, $R = 999$).}
	\label{tab:boston_results}
	\small
	\begin{tabular}{llrrrrr}
		\toprule
		Method & Type & Observed & Null mean & Null SD & $Z$ & $p$ \\
		\midrule
		nLoI      & div & 0.176  & 0.211  & 0.0004 & $-85.4$    & $< 0.001$ \\
		Hellinger & div & 0.373  & 0.443  & 0.0008 & $-90.9$    & $< 0.001$ \\
		wRMSE     & div & 0.892  & 0.924  & 0.0003 & $-92.3$    & $< 0.001$ \\
		RV        & sim & 0.518  & 0.163  & 0.004  & $+85.0$    & $< 0.001$ \\
		SSIM      & sim & 0.220  & 0.015  & 0.001  & $+272.0$   & $< 0.001$ \\
		\bottomrule
	\end{tabular}
\end{table}

The Boston Housing results demonstrate that all divergence measures remain in their theoretical $[0,1]$ range: the nLoI of 0.176, Hellinger distance of 0.373, and wRMSE of 0.892 are all bounded as expected. Interestingly, the nLoI and Hellinger values are comparable to the mtcars case, suggesting similar per-pair reconstruction quality despite the much larger matrix ($M = 62{,}481$ vs.\ $231$). The RV coefficient of 0.518, while lower than Iris (0.960) and mtcars (0.845), still indicates substantial structural agreement and is highly significant.

\subsection{Cross-Dataset Summary}

Table~\ref{tab:cross_dataset} provides a comparative summary across all three datasets.

\begin{table}[h!]
	\centering
	\caption{Cross-dataset comparison of E2Tree reconstruction quality.}
	\label{tab:cross_dataset}
	\small
	\begin{tabular}{lrrrrrrr}
		\toprule
		Dataset & $n$ & Task & nLoI & Hellinger & RV & SSIM & All $p$ \\
		\midrule
		Iris           & 105 & classif. & 0.049 & 0.211 & 0.960 & 0.582 & $< 0.001$ \\
		mtcars         &  22 & regr.    & 0.144 & 0.344 & 0.845 & 0.500 & $< 0.001$ \\
		Boston Housing & 354 & regr.    & 0.176 & 0.373 & 0.518 & 0.220 & $< 0.001$ \\
		\bottomrule
	\end{tabular}
\end{table}

The results are consistent across all three settings: every measure detects significant reconstruction ($p < 0.001$), and the divergence measures remain within their theoretical $[0,1]$ bounds, confirming the normalization is well-calibrated. Reconstruction quality follows the task structure: the Iris classification problem, with well-separated classes, yields the highest fidelity (nLoI $= 0.049$, Hellinger $= 0.211$), while Boston Housing presents the most challenging regression surface (nLoI $= 0.176$, Hellinger $= 0.373$). The RV coefficient tracks the same ordering (0.960, 0.845, 0.518) and proves particularly informative for ranking global structural agreement across datasets.

\section{Discussion}
\label{sec:discussion}

\subsection*{Agreement vs.\ Association: Why the Distinction Matters}

E2Tree validation is, at its core, a problem of \emph{agreement}, not of \emph{association}. The distinction, long settled in clinical method comparison \citep{Bland1986,Lin1989}, extends beyond that domain. \citet{GandyMatcham2025} demonstrate an analogous failure mode in survival analysis, showing that standard concordance indices can systematically prefer an inferior model when hazard functions cross, a direct consequence of measuring rank-based association rather than proper predictive agreement. The lesson generalizes: a measure that is scale-invariant or rank-based will declare ``better'' a model that is actually worse in absolute terms. The present work advances the same argument for proximity matrix comparison.

Consider two scenarios for an E2Tree with $K$ terminal nodes applied to a classification problem with three well-separated classes:

\begin{itemize}
	\item \textbf{Scenario A}: The E2Tree correctly identifies all three classes. Within each terminal node, the ensemble proximity is $o_{ij} \approx 0.8$, while the E2Tree assigns $\hat{o}_{ij} = 1$. Between nodes, $o_{ij} \approx 0.1$ and $\hat{o}_{ij} = 0$. The Mantel correlation is very high ($r > 0.95$) because the pattern is preserved. But the nLoI reveals a non-trivial $LoI_\text{in}$ component: the within-node values are systematically overestimated.

	\item \textbf{Scenario B}: The E2Tree merges two similar classes into one terminal node. Within the merged node, the ensemble proximity is bimodal ($o_{ij} \approx 0.8$ for same-class pairs, $o_{ij} \approx 0.4$ for cross-class pairs), while $\hat{o}_{ij} = 1$ for all. The Mantel correlation may still be high (the broad pattern is preserved), but the nLoI detects a larger $LoI_\text{in}$, and, crucially, the decomposition reveals \emph{why}: the problem is within-node heterogeneity, not between-node separation.
\end{itemize}

In both scenarios, the Mantel test would declare ``significant association'' with $p < 0.001$. The nLoI with its decomposition tells the practitioner \emph{how much} the reconstruction deviates and \emph{where} the deviation occurs. This is the difference between a diagnostic tool and a pass/fail test.

Table~\ref{tab:goi_interpretation} summarizes the key indicators for a comprehensive assessment.

\begin{table}[h!]
	\centering
	\caption{Assessment criteria for E2Tree validation.}
	\label{tab:goi_interpretation}
	\begin{tabular}{lll}
		\toprule
		Criterion & Favorable outcome & Interpretation \\
		\midrule
		$p$-value & $< 0.05$ & Reconstruction significantly better than chance \\
		$Z$-score & Large $|Z|$ & Strong departure from null \\
		CI & Separated from null & Reliable reconstruction \\
		Divergence (nLoI, $H$) & Close to 0 & High reconstruction fidelity \\
		$\bar{\ell}_\text{out}$ & $< 0.1$ & Partition correctly separates low-proximity pairs \\
		$\bar{\ell}_\text{in}$ & $< 0.01$ & Within-node values well-calibrated \\
		\bottomrule
	\end{tabular}
\end{table}

\subsection*{The Decomposition as a Diagnostic Tool}

The decomposability of the LoI (Proposition~\ref{prop:decomposability}) is the primary methodological contribution that distinguishes our approach from correlation-based alternatives. As discussed in Remark~\ref{rem:decomposability}, the raw totals $LoI_\text{in}$ and $LoI_\text{out}$ are not directly comparable due to the different cardinalities of within-node and between-node pair sets. The per-pair averages $\bar{\ell}_\text{in}$ and $\bar{\ell}_\text{out}$ (equation~\ref{eq:mean_in_out}) provide the appropriate diagnostic quantities:

\begin{itemize}
	\item $\bar{\ell}_\text{out}$ (average ensemble proximity lost per separated pair): values below 0.1 indicate the partition correctly separates low-proximity pairs; values above 0.3 indicate the E2Tree is splitting apart pairs with substantial ensemble proximity.
	\item $\bar{\ell}_\text{in}$ (average calibration error per grouped pair): values below 0.01 indicate excellent within-node reconstruction; values above 0.1 indicate substantial calibration discrepancy.
\end{itemize}

No association-based measure (Mantel, RV, or otherwise) can provide this type of directional guidance.

\subsection*{Choosing Among Measures}

Since the simulation results confirm that all measures achieve comparable statistical power (Section~\ref{sec:sim_power}), the choice among measures should be guided by the analytical question rather than by power considerations:

\begin{itemize}
	\item \textbf{nLoI and decomposition}: When the goal is to \emph{diagnose} reconstruction quality and guide model improvement. The per-pair averages $\bar{\ell}_\text{in}$ and $\bar{\ell}_\text{out}$ are unique to the nLoI.

	\item \textbf{Hellinger distance}: When a single bounded, interpretable number is needed for reporting. Its range $[0,1]$, metric properties, and robustness to sparsity make it the most portable summary statistic.

	\item \textbf{RV coefficient}: When the question is ``does the E2Tree preserve the ensemble's structural pattern?'' rather than ``does it reproduce the values?'' The RV is the appropriate measure when scale differences are expected and acceptable.

	\item \textbf{SSIM}: When block structure preservation is of primary interest, e.g., evaluating whether cluster boundaries are correctly identified.
\end{itemize}

We recommend reporting the nLoI with its decomposition alongside the Hellinger distance: the nLoI provides diagnostics, and the Hellinger provides a clean summary. The permutation $p$-value and $Z$-score should always accompany the point estimates.

\subsection*{Relationship to Alternative Approaches}

  The family of measures proposed here is not the only approach to comparing two proximity matrices. \citet{Xu2025} develops asymptotic distribution-free tests based on Maximum Mean Discrepancy (MMD) for two-sample comparisons, which establish Gaussian limiting distributions and avoid the computational cost of permutation resampling entirely.
While the MMD framework addresses a related but distinct inferential question (whether two samples are drawn from the same distribution, rather than whether a particular reconstruction matches a particular target), it represents a complementary approach that could in principle be adapted to the E2Tree validation setting. An alternative two-sample strategy based on random forests has been proposed by \citet{Hediger2022}, who exploit the out-of-bag error of a classifier trained to distinguish between two samples as a test statistic; this
approach is particularly natural in the E2Tree context, where the ensemble itself is a random forest, and the proximity matrices $O$ and $\hat{O}_T$ encode the same observations under different grouping structures. A formal comparison between the permutation-based approach developed here and these distribution-free alternatives (both MMD and random-forest-based) is a natural direction for future work.

\section{Conclusion}
\label{sec:conclusion}

Validating E2Tree reconstruction quality requires measuring \emph{agreement} rather than mere \emph{association} between the ensemble and E2Tree proximity matrices. The Mantel test and other correlation-based approaches can confirm that two matrices share a common pattern, but they cannot quantify how closely the reconstructed values match the originals, nor can they identify the source of reconstruction failures. The measures proposed here fill that gap.

The primary methodological contribution is the normalized Loss of Interpretability (nLoI), a robustified Neyman-type statistic connected to the Cressie--Read power divergence family ($\lambda = -2$). Its defining feature is \emph{decomposability} into within-node ($LoI_\text{in}$) and between-node ($LoI_\text{out}$) components. Because the raw totals are not directly comparable, since within-node and between-node pair sets differ vastly in cardinality, we work instead with the per-pair averages $\bar{\ell}_\text{in}$ and $\bar{\ell}_\text{out}$. These two quantities tell the practitioner whether reconstruction failure originates in suboptimal partition boundaries ($\bar{\ell}_\text{out}$ large) or in poor within-node calibration ($\bar{\ell}_\text{in}$ large). No scalar association measure can provide this directional guidance.

The nLoI is complemented by four measures (Hellinger distance, wRMSE, RV coefficient, and SSIM), each targeting a distinct aspect of matrix agreement. The simulation study shows that all measures achieve comparable power and correct Type~I error control; the choice among them is therefore guided by the analytical question, not by statistical efficiency. For general use, the Hellinger distance is our recommended default: its bounded range, metric properties, and robustness to sparsity make it the most portable single-number summary.

Underpinning all five measures is a unified permutation testing framework based on simultaneous row/column permutation, which delivers valid inference for the entire family within a single computational pass.

Three extensions merit attention. The first, the asymptotic behavior of the measures as $n \to \infty$ and $B \to \infty$ remains an open theoretical question. Then, the per-pair averages $\bar{\ell}_\text{in}$ and $\bar{\ell}_\text{out}$ suggest a natural route toward automated hyperparameter selection for E2Tree, where the decomposition could drive tree complexity in a principled way. Finally, extending the framework to gradient boosting, XGBoost, and other interpretable surrogates beyond E2Tree would establish it as a general-purpose tool for approximation fidelity in ensemble learning.

\subsection*{Author Contributions}

Massimo Aria: Conceptualization, Methodology, Software, Writing -- review and editing, Project administration.
Agostino Gnasso: Conceptualization, Methodology, Software, Formal analysis, Writing -- Original draft preparation.
Carmela Iorio: Conceptualization, Methodology, Formal analysis, Writing -- review and editing.

\bibliographystyle{plainnat}

\bibliography{loi_ref}

\end{document}